%% file: LilUCB.tex
\documentclass[12pt]{article}

\usepackage[colorlinks]{hyperref}            
\usepackage{color}
\usepackage{graphicx,amsmath,amssymb,amsfonts,bm,cite,epsfig,epsf,url,alg,dsfont}
\usepackage{times}
\usepackage{bbm}      
\usepackage{booktabs}
\usepackage{cases}
\usepackage{fullpage}
\usepackage{caption}
\usepackage[top=1in,bottom=1in,left=1in,right=1in]{geometry}
\usepackage{float}
\usepackage{subcaption}
\usepackage{rotating}

\input{Commands}

\newcommand{\hatmui}{\widehat{\mu}_{i,T_i(t)}}

\newcommand{\Hn}{\mathbf{H}}

\newtheorem{theorem}{Theorem}
\newcommand{\BlackBox}{\rule{1.5ex}{1.5ex}}  
\newenvironment{proof}{\par\noindent{\bf Proof\ }}{\hfill\BlackBox\\[2mm]}

\begin{document}

\title{lil' UCB : An Optimal Exploration Algorithm for Multi-Armed Bandits}
\date{}
\author{Kevin Jamieson$^\dagger$, Matthew Malloy$^\dagger$, Robert Nowak$^\dagger$, and S\'{e}bastien Bubeck$^{\ddagger}$  \\
$^\dagger$Department of Electrical and Computer Engineering, \\  University of Wisconsin-Madison \\
$^\ddagger$Department of Operations Research and Financial Engineering, \\
Princeton University  }	
\maketitle
		
\abstract{The paper proposes a novel upper confidence bound (UCB) procedure for identifying the arm with the largest mean in a multi-armed bandit game in the fixed confidence setting using a small number of total samples.  The procedure cannot be improved in the sense that the number of samples required to identify the best arm is within a constant factor of a lower bound based on the law of the iterated logarithm (LIL). Inspired by the LIL, we construct our confidence bounds to explicitly account for the infinite time horizon of the algorithm. 
In addition, by using a novel stopping time for the algorithm we avoid a union bound over the arms that has been observed in other UCB-type algorithms. We prove that the algorithm is optimal up to constants and also show through simulations that it provides superior performance with respect to the state-of-the-art. }
		
\section{Introduction}

This paper introduces a new algorithm for the \emph{best arm} problem in the stochastic multi-armed bandit (MAB) setting. Consider a MAB  with $n$ arms, each with unknown mean payoff $\mu_1, \dots, \mu_n$ in $[0,1]$. A sample of the $i$th arm is an independent realization of a sub-Gaussian random variable with mean $\mu_i$. In the \emph{fixed confidence setting}, the goal of the best arm problem is to devise a sampling procedure with a single input $\delta$ that, regardless of the values of $\mu_1,\dots, \mu_n$, finds the arm with the largest mean with probability at least $1-\delta$. More precisely, best arm procedures must satisfy $\sup_{\mu_1,\dots,\mu_n} \P(\widehat{i} \neq i^*) \leq \delta$, where $i^*$ is the best arm, $\widehat{i}$ an estimate of the best arm, and the supremum is taken over all set of means such that there exists a unique best arm.   In this sense, best arm procedures must automatically adjust sampling to ensure success when the mean of the best and second best arms are arbitrarily close. 
Contrast this with the {\em fixed budget setting} where the total number of samples remains a constant and the confidence in which the best arm is identified within the given budget varies with the setting of the means. While the fixed budget and fixed confidence settings are related (see \cite{gabillon2012best} for a discussion) this paper focuses on the fixed confidence setting only. 


The best arm problem has a long history dating back to the '50s with the work of \cite{paulson1964sequential, bechhofer1958sequential}.  In the fixed confidence setting, the last decade has seen a flurry of activity providing new upper and lower bounds.  In 2002, the {\em successive elimination} procedure of \cite{even2002pac} was shown to find the best arm with order $\sum_{i\neq i^*} \Delta_i^{-2} \log( n \Delta_i^{-2})$ samples, where $\Delta_i = \mu_{i^*} - \mu_i$, coming within a logarithmic factor of the lower bound of  $\sum_{i\neq i^*} \Delta_i^{-2}$, shown in 2004 in \cite{mannor2004sample}.  A similar bound was also obtained using a procedure known as {\em LUCB1} that was originally designed for finding the $m$-best arms \cite{kalyanakrishnan2012pac}. Recently, \cite{jamieson2013finding} proposed a procedure called {\em PRISM} which succeeds with  $\sum_i \Delta_i^{-2} \log \log \left( \sum_j \Delta_j^{-2}\right) $ or $\sum_i \Delta_i^{-2}  \log \left(  \Delta_i^{-2}\right) $ samples depending on the parameterization of the algorithm, improving the result of \cite{even2002pac} by at least a factor of $\log(n)$.  The best sample complexity result for the fixed confidence setting comes from a procedure similar to PRISM, called {\em exponential-gap elimination} \cite{karnin2013almost}, which guarantees identification of the best arm with high probability using order $\sum_i \Delta_i^{-2} \log \log \Delta_i^{-2}$ samples, coming within a doubly logarithmic factor of the lower bound of \cite{mannor2004sample}. While the authors of \cite{karnin2013almost} conjecture that the $\log \log$ term cannot be avoided, it remained unclear as to whether the upper bound of \cite{karnin2013almost} or the lower bound of \cite{mannor2004sample} was loose. 

The classic work of \cite{Farrell} answers this question.  It shows that the doubly logarithmic factor is necessary, implying that order  $\sum_i \Delta_i^{-2} \log \log \Delta_i^{-2}$ samples are necessary and sufficient in the sense that no procedure can satisfy $\sup_{\Delta_1,\dots, \Delta_n} \P(\widehat{i}\neq  i^* )\leq \delta$ and use fewer than $\sum_i \Delta_i^{-2} \log \log \Delta_i^{-2}$ samples in expectation for all $\Delta_1,\dots, \Delta_n$. 
The doubly logarithmic factor is a consequence of the law of the iterated logarithm (LIL) \cite{darling1985iterated}. The LIL states that if $X_{\ell}$ are i.i.d.\ sub-Gaussian random variables with $\E[X_\ell] = 0$, $\E[X_\ell^2]=\sigma^2$ and we define $S_t = \sum_{\ell=1}^t X_\ell$ then 
\begin{align} \nonumber
\limsup_{t \rightarrow \infty} \frac{S_t}{\sqrt{2 \sigma^2 t \log\log(t)}} =1   \mbox{\ \ and \ \ }
\liminf_{t \rightarrow \infty} \frac{S_t}{\sqrt{2 \sigma^2 t \log\log(t)}} =-1  
\end{align}
almost surely. Here is the basic intuition behind the lower bound.  Consider the two-arm problem and let $\Delta$
be the difference between the means.  In this case, it is reasonable to sample both arms equally and consider the sum of differences of the samples, which is a random walk with drift $\Delta$.  The deterministic drift crosses the LIL bound (for a zero-mean walk) when $t\, \Delta = \sqrt{2t\log\log t}$.  Solving this equation for $t$ yields $t \approx 2\Delta^{-2} \log \log \Delta^{-2}$.  This intuition will be formalized in the next section.

The LIL also motivates a novel approach to the best arm problem. Specifically, the LIL suggests a natural scaling for confidence bounds on empirical means, and we follow this intuition to develop a new algorithm for the best-arm problem.  The algorithm is an Upper Confidence Bound (UCB) procedure \cite{auer2002finite} based on a finite sample version of the LIL.
The new algorithm, called lil'UCB, is described in Figure~\ref{fig:lilucb}. By explicitly accounting for the $\log \log$ factor in the confidence bound and using a novel stopping criterion, our analysis of lil'UCB avoids taking naive union bounds over time, as encountered in some UCB algorithms \cite{kalyanakrishnan2012pac,audibert2010best}, as well as the wasteful ``doubling trick''  often employed in algorithms that proceed in epochs, such as the PRISM and exponential-gap elimination procedures \cite{even2002pac,karnin2013almost,jamieson2013finding}. Also, in some analyses of best arm algorithms the upper confidence bounds of each arm are designed to hold with high probability for all arms uniformly, incurring a $\log(n)$ term in the confidence bound as a result of the necessary union bound over the $n$ arms \cite{even2002pac,kalyanakrishnan2012pac,audibert2010best}. However, our stopping time allows for a tighter analysis so that arms with larger gaps are allowed larger confidence bounds than those arms with smaller gaps where higher confidence is required.   Like exponential-gap elimination, lil'UCB is order optimal in terms of sample complexity. 

One of the main motivations for this work was to develop an algorithm that exhibits great practical performance in addition to optimal sample complexity.  While the sample complexity of exponential-gap elimination 
is optimal up to constants, and PRISM up to small $\log \log$ factors, the empirical performance of these methods is rather disappointing, even when compared to non-sequential sampling.  Both PRISM and exponential-gap elimination employ {\em median elimination} \cite{even2002pac} as a subroutine.  Median elimination is used to find an arm that is within $\epsilon>0$ of the largest, and has sample complexity within a constant factor of optimal for this subproblem.  However, the constant factors tend to be quite large, and repeated applications of median elimination within PRISM and exponential-gap elimination are extremely wasteful.  
On the contrary, lil'UCB does not invoke wasteful subroutines.  As we will show, in addition to having the best theoretical sample complexities bounds known to date, lil'UCB exhibits superior performance in practice with respect to state-of-the-art algorithms.

\section{Lower Bound}

Before introducing the lil'UCB algorithm, we show that the $\log\log$ factor  in the sample complexity is necessary for best-arm identification. It suffices to consider a two armed bandit problem with a gap $\Delta$.  If a lower bound on the gap is unknown, then the $\log\log$ factor is necessary, as shown by the following result of \cite{Farrell}.

\begin{corollary} \label{thm:lim}
Consider the best arm problem in the fixed confidence setting with $n=2$ and expected number of samples $\E_{\Delta}[T]$.  Any procedure with $\sup_{\Delta \neq 0} \; \mathbb{P}(\widehat{i} \neq i^* ) \leq \delta$, $\delta \in (0,1/2)$, necessarily has 
\begin{eqnarray} \label{eqn:thm:2}   \nonumber
\limsup_{{\Delta} \rightarrow 0} \frac{\E_{\Delta} [T]}{\Delta^{-2} \log \log {\Delta^{-2}} } &\geq& 2-4 \delta. 
\end{eqnarray}
\end{corollary}
\begin{proof}
Consider a reduction of the best arm problem with $n=2$ in which the value of one arm is known.  In this case, the only strategy available is to sample the other arm some number of times to determine if it is less than or greater than the known value.   We have reduced the problem precisely to that studied by Farrell in \cite{Farrell}, restated below.
\end{proof}

\begin{theorem} \cite[Theorem 1]{Farrell}. \label{thm:3} Let $X_i \overset{i.i.d.}{\sim} \mathcal{N}(\Delta,1)$, where $\Delta\neq 0$ is unknown.  Consider testing whether $\Delta>0$ or $\Delta <0$. Let $Y \in \{-1,1\}$ be the decision of any such test based on $T$ samples (possibly a random number) and let $\delta \in (0,1/2)$.  If $\sup_{\Delta\neq 0}  \P(Y\neq \mbox{sign}(\Delta) ) \leq \delta$, 
then 
\begin{eqnarray} \nonumber
\limsup_{\Delta\rightarrow 0} \frac{\E_{\Delta}[T]}{\Delta^{-2} \log \log {\Delta^{-2}} } \geq 2-4 \delta .
\end{eqnarray}
\end{theorem}

Corollary \ref{thm:lim} implies that in the fixed confidence setting, no best arm procedure can  have $\sup   \P(\widehat{i} \neq i^* ) \leq \delta$ \emph{and} use fewer than  $(2 - 4\delta) \sum_i \Delta_i^{-2} \log \log \Delta_i^{-2}$ samples in expectation for all $\Delta_i$. 

In brief, the result of Farrell follows by studying the form of a known optimal test, termed a generalized sequential probability ratio test, which compares the running empirical mean of $X$ after $t$ samples against a series of thresholds.  In the limit as $t$ increases, if the thresholds are not at least $\sqrt{(2/t) \log\log(t)}$ then the LIL implies the procedure will fail with probability approaching 1/2 for small values of $\mu$.  Setting the thresholds to be just greater than $\sqrt{(2/t) \log\log(t)}$, in the limit, one can show the expected number of samples must scale as $\Delta^{-2} \log \log {\Delta^{-2}}$. 

The proof in \cite{Farrell} is quite involved; to make this paper more self-contained we provide a short argument for a slightly simpler result than above in Appendix~\ref{app:LB}.

\newpage

\section{Procedure}
This section introduces lil'UCB.  The procedure operates by sampling the arm with the largest upper confidence bound; the confidence bounds are defined to account for the implications of the LIL. The procedure terminates when an arm has been sampled more than a constant fraction of the total number of samples. 
Fig. \ref{fig:lilucb} details the algorithm and Theorem \ref{th:lilucb} quantifies performance. In what follows, let $X_{i,s}$, $s=1,2,\dots$ denote independent samples from arm $i$ and let $T_i(t)$ denote the number of times arm $i$ has been sampled up to time $t$. Define $\hatmui :=  \frac{1}{T_i(t)} \sum_{s=1}^{T_i(t)}X_{i,s}$ to be the empirical mean of the $T_i(t)$ samples from arm $i$ up to time $t$. 

\begin{figure}[h]
\centerline{
\fbox{\parbox[b]{5.35in}{{\underline{\bf lil' UCB}}  \\ \vspace{-.05in} \\ \small
{\bf input}: confidence $\delta>0$, algorithm parameters $\varepsilon$, $a$, $\beta >0$ \\
{\bf initialize}: sample each arm once, set $T_i(t) =1$ for all $i$ and set $t=n$ \\
{\bf while} $T_i(t) < 1 + a \sum_{j \neq i} T_j(t) \mbox{ for all $i$}$ \\ \\
\indent \hspace{.4cm} sample arm \vspace{-.5cm}
\begin{eqnarray*}
\qquad I_t  & = & \argmax_{i \in \{1,\dots,n\}} \, \left\{ \hatmui + (1+\beta) (1+\sqrt{\eps}) \sqrt{ \frac{2 (1+\eps) \log \left(\frac{\log((1+\eps) T_i(t))}{\delta}\right)}{T_i(t)}} \right\}. \\ \ \end{eqnarray*}
\hspace{.4cm} set  \vspace{-.4cm}
\begin{eqnarray*}
& & \hspace{-7.3cm}  T_i(t+1)  =\left\{\begin{array}{ll} T_i(t)+1 & i=I_t \\ \ \\
T_i(t) & i \neq I_t \end{array}\right. \hfill \\ \
& & \hspace{-7.3cm}  t= t+1.  \hfill
\end{eqnarray*}
{\bf else} stop and output $\arg \max_{i \in \{1,\dots,n\}} T_i(t)$ 
}}}
\caption{\label{fig:lilucb} lil' UCB}
\end{figure}







\ \\ 
Define
$$\Hn_1 = \sum_{i \neq i^*} \frac{1}{\Delta_i^2} \ \  \ \ \text{and} \ \  \ \ \Hn_3 = \sum_{i \neq i^*} \frac{\log(\log(c/\Delta_i^2))}{\Delta_i^2} $$
where $c>0$ is a constant that appears in the analysis that makes the $\log\log$ term well defined for all $\Delta_i \in (0,1]$. Our main result is the following.
\begin{theorem} \label{th:lilucb}
For any $\epsilon, \beta >0$, $\delta \in (0,\log(1+\epsilon)/e)$\footnote{The range on $\delta$ is restricted to guarantee that $\log (\frac{\log((1+\eps) t )}{\delta})$ is well defined. This makes the analysis cleaner but in practice one can allow the full range of $\delta$ by using $\log (\frac{\log((1+\eps) t+2 )}{\delta})$ instead and obtain the same theoretical guarantees. \label{fnt:delta_foot}} and $$a \geq \frac{    1+ \frac{\log\left(2 \log\left( \left( \frac{2+\beta}{\beta} \right)^2 /\delta\right) \right) }{\log(1/\delta) }    }{1-\delta-\sqrt{\sqrt{\delta} \log(1/\delta)}}  \left(\frac{2+\beta}{\beta}\right)^2,$$
with probability at least $1-\sqrt{\rho \delta} - \frac{4 \rho\delta}{1-\rho\delta}$, lil' UCB stops after at most  $c_1 \Hn_1 \log(1/\delta) + c_3 \Hn_3 $ samples and outputs the optimal arm where $\rho = \frac{2+\eps}{\eps} \left(\frac{1}{\log(1+\eps)}\right)^{1+\epsilon}$ and $c_1,c_3 >0$ are constants that depend only on $\epsilon, \beta$. 
\end{theorem}

Note that regardless of the choice of $\eps,\beta$ the algorithm obtains the optimal query complexity of $ \mathbf{H}_1 \log(1/\delta) + \mathbf{H}_3$ up to constant factors. However, in practice some settings of $\epsilon,\beta$ perform better than others. We observe from the bounds in the proof that the optimal choice for the exploration constant is $\beta \approx 1.66$ but we suggest using $\beta =1$ and $a=\left( \frac{\beta+2}{\beta} \right)^2$. The optimal value for $\epsilon$ is less evident as it depends on $\delta$ but we suggest using $\epsilon=0.01$. If one is willing to forego theoretical guarantees, we recommend taking a more aggressive setting with $\epsilon=0$, $\beta=0.5$, and $a = 1+10/n$ which is motivated by simulation results presented later. We prove the theorem via two lemmas, one for the total number of samples and one for the correctness of the algorithm. In the lemmas we give precise constants.

\section{Proof of Theorem \ref{th:lilucb}} \label{sec:analysis}
Before stating the two main lemmas that imply the result, we first present a finite form of the law of iterated logarithm. This finite LIL bound is necessary for our analysis and may also prove useful for other applications.

\begin{lemma} \label{lem:lil}
Let $X_1, X_2, \ldots$ be i.i.d. centered sub-Gaussian\footnote{A random variable $X$ is said to be sub-Gaussian with scale parameter $\sigma$ if for all $t \in \R$ we have $\E[ \exp\{tX\} ] \leq \exp\{ \sigma^2 t^2 / 2 \}$.} random variables with scale parameter $\sigma$. For any $\epsilon \in (0,1)$ and $\delta \in (0,\log(1+\epsilon)/e)$\footnote{See footnote \ref{fnt:delta_foot}} one has with probability at least $1-\frac{2+\eps}{\eps} \left(\frac{\delta}{\log(1+\eps)}\right)^{1+\epsilon}$ for any $t \geq 1$,
$$\sum_{s=1}^t X_s \leq (1+\sqrt{\eps}) \sqrt{2 \sigma^2 (1+\eps) t \log \left(\frac{\log((1+\eps) t )}{\delta}\right)} .$$
\end{lemma}

\begin{proof}
We denote $S_t = \sum_{s=1}^t X_s$, and 
$\psi(x) = \sqrt{2 \sigma^2 x \log \left(\frac{\log(x)}{\delta}\right)}$. 
We also define by induction the sequence of integers $(u_k)$ as follows: $u_0=1$, $u_{k+1} = \lceil (1+\epsilon) u_k \rceil$.
\newline

\noindent
\textbf{Step 1: Control of $S_{u_k}, k \geq 1$.} 
The following inequalities hold true thanks to an union bound together with Chernoff's bound, the fact that $u_k \geq (1+\epsilon)^k$, and a simple sum-integral comparison:
\begin{eqnarray*}
\P \left(\exists k \geq 1 : S_{u_k} \geq \sqrt{1+\eps} \ \psi(u_k) \right) & \leq & \sum_{k=1}^{\infty} \exp\left( - (1+\eps) \log \left(\frac{\log(u_k)}{\delta}\right)\right) \\
& \leq & \sum_{k=1}^{\infty} \left(\frac{\delta}{k \log(1+\eps)}\right)^{1+\epsilon} \\
& \leq & \left(1+\frac{1}{\eps}\right) \left(\frac{\delta}{\log(1+\eps)}\right)^{1+\epsilon} .
\end{eqnarray*}
\newline

\noindent
\textbf{Step 2: Control of $S_{t}, t \in (u_k, u_{k+1})$.} 
Recall that Hoeffding's maximal inequality\footnote{It is an easy exercise to verify that Azuma-Hoeffding holds for martingale differences with sub-Gaussian increments, which implies Hoeffding's maximal inequality for sub-Gaussian distributions.} states that for any $m \geq 1$ and $x>0$ one has
$$\P(\exists \ t \in [m] \ \text{s.t.} \ S_t \geq x) \leq \exp \left(-\frac{x^2}{2 \sigma^2 m}\right).$$ 
This implies that the following inequalities hold true (by using trivial manipulations on the sequence $(u_k)$):
\begin{align*}
& \P \left(\exists \ t \in \{u_k+1, \hdots, u_{k+1} - 1\} : S_{t} - S_{u_k} \geq \sqrt{\epsilon} \ {\psi(u_{k+1})} \right) \\
& =
\P \left(\exists \ t \in [u_{k+1} -u_k - 1] : S_{t} \geq \sqrt{\epsilon} \ {\psi(u_{k+1})} \right) \\
& \leq \exp\left( - \epsilon \frac{u_{k+1}}{u_{k+1} -u_k - 1}\log \left(\frac{\log(u_{k+1})}{\delta}\right)\right) \\
& \leq \exp\left( - (1+\epsilon) \log \left(\frac{\log(u_{k+1})}{\delta}\right)\right) \\
& \leq \left(\frac{\delta}{(k+1)\log(1+\eps)}\right)^{1+\epsilon} .
\end{align*}
\newline

\noindent
\textbf{Step 3:} By putting together the results of Step 1 and Step 2 we obtain that with probability at least $1-\frac{2+\eps}{\eps} \left(\frac{\delta}{\log(1+\eps)}\right)^{1+\epsilon}$,
one has for any $k \geq 0$ and any $t \in \{u_k+1, \hdots, u_{k+1} \}$,
\begin{eqnarray*}
S_t & = & S_t - S_{u_k} + S_{u_k} \\
& \leq & \sqrt{\eps} \ \psi(u_{k+1}) + \sqrt{1+\eps} \ \psi(u_k) \\
& \leq & \sqrt{\eps} \ \psi((1+\eps) t) + \sqrt{1+\eps} \ \psi(t) \\
& \leq & (1+\sqrt{\eps}) \ \psi((1+\eps) t),
\end{eqnarray*}
which concludes the proof.
\end{proof}

Without loss of generality we assume that $\mu_1 > \mu_2 \geq \hdots \geq \mu_n$. To shorten notation we denote
$$U(t,\omega) = (1+\sqrt{\eps}) \sqrt{\frac{2 (1+\eps)}{t} \log \left(\frac{\log((1+\eps) t)}{\omega}\right)}.$$
The following events will be useful in the analysis:
$$\cE_i(\omega) = \{\forall t \geq 1, |\widehat{\mu}_{i,t} - \mu_i| \leq U(t,\omega)\} $$
where $\widehat{\mu}_{i,t} = \frac{1}{t}\sum_{j=1}^t x_{i,j}$.
Note that Lemma \ref{lem:lil} shows $\P(\cE_i(\omega)) = O(\omega)$. The following trivial inequalities will also be useful (the second one is derived from the first inequality and the fact that $\frac{x+a}{x+b} \leq \frac{a}{b}$ for $a \geq b$, $x \geq 0$). For $t \geq 1$,
\begin{align} \label{eq:stupidinequality}
&  \frac{1}{t} \log \left(\frac{\log((1+\eps) t)}{\omega}\right) \geq c \Rightarrow t \leq \frac{1}{c} \log \left(\frac{2 \log ((1+\eps)/(c\omega))}{\omega} \right) , \\
\label{eq:stupidinequality2}
&  \frac{1}{t} \log \left(\frac{\log((1+\eps) t)}{\omega}\right) \geq \frac{c}{s} \log \left(\frac{\log((1+\eps) s)}{\delta}\right) \ \text{and} \ \omega \leq \delta \Rightarrow t \leq \frac{s}{c} \frac{\log\left(2 \log\left( \frac{1}{c \omega } \right) / \omega \right)}{\log(1/\delta)} . 
\end{align}

\begin{lemma} \label{lem:1}
Let $\gamma = 2(2+\beta)^2 (1+\sqrt{\eps})^2 (1+\eps)$ and $\rho = \frac{2+\eps}{\eps} \left(\frac{1}{\log(1+\eps)}\right)^{1+\epsilon}$. With probability at least $1- 2\rho \delta$ one has for any $t\geq1$,
$$\sum_{i=2}^n T_i(t) \leq n + \gamma 8e  \Hn_1 \log(1/\delta) + \sum_{i=2}^n \gamma \frac{\log(2\log(\gamma (1+\eps)/\Delta_i^2))}{\Delta_i^2}.$$
\end{lemma}

\begin{proof}
We decompose the proof in two steps.
\newline

\noindent
\textbf{Step 1.} Let $i > 1$. Assuming that $\cE_1(\delta)$ and $\cE_i(\omega)$ hold true and that $I_t = i$ one has
$$\mu_i + U(T_i(t), \omega) + (1+\beta) U(T_i(t), \delta) \geq \hatmui + (1+\beta) U(T_i(t), \delta) \geq \widehat{\mu}_{1,T_1(t)} + (1+\beta) U(T_1(t), \delta) \geq \mu_1 ,$$
which implies $(2+\beta) U(T_i(t), \min(\omega, \delta)) \geq \Delta_i$. Thus using \eqref{eq:stupidinequality} with $c = \frac{\Delta_i^2}{2(2+\beta)^2 (1+\sqrt{\eps})^2 (1+\eps)}$ one obtains that if $\cE_1(\delta)$ and $\cE_i(\omega)$ hold true and $I_t = i$ then
\begin{eqnarray*}
T_i(t) & \leq & \frac{2(2+\beta)^2 (1+\sqrt{\eps})^2 (1+\eps)}{\Delta_i^2} \log \left(\frac{2 \log (2(2+\beta)^2 (1+\sqrt{\eps})^2 (1+\eps)^2/\Delta_i^2 /\min(\omega, \delta))}{\min(\omega, \delta)} \right) \\
& \leq & \tau_i + \frac{\gamma}{\Delta_i^2} \log\left(\frac{ \log(e/\omega) }{\omega}\right)  \leq  \tau_i +  \frac{2\gamma}{\Delta_i^2} \log\left(\frac{ 1 }{\omega}\right) ,
\end{eqnarray*}
where $\gamma = 2(2+\beta)^2 (1+\sqrt{\eps})^2 (1+\eps)$, and $\tau_i = \frac{\gamma}{\Delta_i^2} \log \left(\frac{2 \log(\gamma (1+\eps)/\Delta_i^2)}{\delta} \right)$. 

Since $T_i(t)$ only increases when $I_t$ is played the above argument shows that the following inequality is true for any time $t \geq 1$:
\begin{equation} \label{eq:nonstupidinequality1}
T_i(t) \ds1\{\cE_1(\delta) \cap \cE_i(\omega) \}\leq 1 + \tau_i +  \frac{2\gamma}{\Delta_i^2} \log\left(\frac{1}{\omega}\right) .
\end{equation}
\newline

\noindent
\textbf{Step 2.} We define the following random variable:
$$\Omega_i = \max\{ \omega \geq 0 : \cE_i(\omega) \ \text{holds true} \}.$$
Note that $\Omega_i$ is well-defined and by Lemma \ref{lem:lil} it holds that $\P(\Omega_i < \omega) \leq \rho \omega$ where $\rho = \frac{2+\eps}{\eps} \left(\frac{1}{\log(1+\eps)}\right)^{1+\epsilon}$. Furthermore one can rewrite \eqref{eq:nonstupidinequality1} as 
\begin{equation} \label{eq:nonstupidinequality2}
T_i(t) \ds1\{\cE_1(\delta)\}\leq 1 + \tau_i +  \frac{2 \gamma}{\Delta_i^2} \log\left(\frac{1}{\Omega_i}\right) .
\end{equation}
We use this equation as follows:
\begin{eqnarray}
\P\left( \sum_{i=2}^n T_i(t) > x + \sum_{i=2}^n (\tau_i+1)\right) & \leq &
\rho \delta + \P\left( \sum_{i=2}^n T_i(t) > x + \sum_{i=2}^n (\tau_i+1) \big| \cE_1(\delta)\right) \notag \\
& \leq & \rho \delta + \P\left( \sum_{i=2}^n  \frac{2\gamma}{\Delta_i^2} \log\left(\frac{1}{\Omega_i}\right)  > x \right) . \label{eq:stuff1}
\end{eqnarray}
Let $Z_i =  \frac{2 \gamma}{\Delta_i^2} \log\left(\frac{\rho}{ \Omega_i}\right)$, $i \in [n]$. Observe that these are independent random variables and since $\P(\Omega_i < \omega) \leq \rho \omega$ it holds that
$\P(Z_i > x) \leq \exp(- x / a_i)$ with $a_i =  2\gamma/\Delta_i^2$. Using standard techniques to bound the sum of sub-exponential random variables one directly obtains that
\begin{equation} \label{eq:subexp}
\P\left( \sum_{i=2}^n Z_i \geq x \right) \leq \exp\left( - \min\left\{ \frac{x^2}{8 e^2 \|a\|_2^2}, \frac{x}{4 e \|a\|_\infty } \right\} \right) \leq \exp\left( - \min\left\{ \frac{x^2}{8 e^2 \|a\|_1^2}, \frac{x}{4 e \|a\|_1 } \right\} \right) .
\end{equation} 
Putting together \eqref{eq:stuff1} and \eqref{eq:subexp} with $x=4 e \|a\|_1 \log(1/ (\rho\delta))$ one obtains
$$\P\left( \sum_{i=2}^n T_i(t) > \sum_{i=2}^n \left(\frac{ 8 e \gamma \log(1/\delta) }{\Delta_i^2} + \tau_i+1\right)\right) \leq2 \rho \delta ,$$
which concludes the proof.
\end{proof}

\begin{lemma} \label{lem:2}
Let $\rho = \frac{2+\eps}{\eps} \left(\frac{1}{\log(1+\eps)}\right)^{1+\epsilon}$. If $$a \geq \frac{    1+ \frac{\log\left(2 \log\left( \left( \frac{2+\beta}{\beta} \right)^2 /\delta\right) \right) }{\log(1/\delta) }    }{1-\delta-\sqrt{\sqrt{\delta} \log(1/\delta)}}  \left(\frac{2+\beta}{\beta}\right)^2,$$ then for all 
$i=2,\dots n$ and $t = 1,2,\dots$,
$$T_i(t) < 1 + a \sum_{j \neq i} T_j(t)   $$
with probability at least $1-\sqrt{\rho \delta} - \frac{2 \rho\delta}{1-\rho\delta}$.
\end{lemma}

\begin{proof}
We decompose the proof in two steps.
\newline

\noindent
\textbf{Step 1.} Let $i > j$. Assuming that $\cE_i(\omega)$ and $\cE_j(\delta)$ hold true and that $I_t = i$ one has
\begin{eqnarray*}
\mu_i + U(T_i(t), \omega) + (1+\beta) U(T_i(t), \delta) & \geq & \hatmui + (1+\beta) U(T_i(t), \delta) \\
& \geq & \widehat{\mu}_{j,T_j(t)} + (1+\beta) U(T_j(t), \delta) \\
& \geq & \mu_j + \beta U(T_j(t), \delta),
\end{eqnarray*}
which implies $(2+\beta) U(T_i(t), \min(\omega, \delta)) \geq \beta U(T_j(t), \delta)$. Thus using \eqref{eq:stupidinequality2} with $c = \left(\frac{\beta}{2+\beta} \right)^2$ one obtains that if $\cE_i(\omega)$ and $\cE_j(\delta)$ hold true and $I_t = i$ then
\begin{eqnarray*}
T_i(t) & \leq &  \left(\frac{2+\beta}{\beta} \right)^2  \frac{\log\left(2 \log\left( \left( \frac{2+\beta}{\beta} \right)^2/ \min(\omega, \delta)  \right)/ \min(\omega, \delta) \right)}{\log(1/\delta)}  T_j(t) .
\end{eqnarray*}

Similarly to Step 1 in the proof of Lemma \ref{lem:1} we use the fact that $T_i(t)$ only increases when $I_t$ is played and the above argument to obtain the following inequality for any time $t \geq 1$:
\begin{equation} \label{eq:nonstupidinequality2}
(T_i(t)-1) \ds1\{\cE_i(\omega) \cap \cE_j(\delta)\}\leq   \left(\frac{2+\beta}{\beta} \right)^2 \frac{\left(2 \log\left( \left( \frac{2+\beta}{\beta} \right)^2 /\min(\omega, \delta) \right)  /\min(\omega, \delta) \right)}  {\log(1/\delta)} T_j(t) .
\end{equation}
\newline

\noindent
\textbf{Step 2.} Using \eqref{eq:nonstupidinequality2} with $\omega = \delta^{i-1}$ we see that
\begin{align*}
\ds1\{\cE_i(\delta^{i-1})\} \frac{1}{i-1}\sum_{j=1}^{i-1}  \ds1\{\cE_j(\delta)\} > 1 - \alpha \ \Rightarrow \ (1- \alpha) (T_i(t)-1) \leq    \kappa \sum_{j\neq i} T_j(t)
\end{align*}
where $\kappa = \left(\frac{2+\beta}{\beta} \right)^2 \left(1+ \frac{\log\left(2 \log\left( \left( \frac{2+\beta}{\beta} \right)^2 /\delta\right) \right) }{\log(1/\delta) }  \right)$. This implies the following, using that $\P(\cE_i(\omega)) \geq 1 - \rho \omega$,
\begin{align*}
& \P\left( \exists \ (i,t) \in \{2, \hdots, n\} \times \{1,\dots\} : (1- \alpha) (T_i(t)-1) \geq \kappa  \sum_{j\neq i} T_j(t)\right) \\
& \leq \P\left( \exists \ i \in \{2, \hdots, n\} : \ds1\{\cE_i(\delta^{i-1})\} \frac{1}{i-1}\sum_{j=1}^{i-1} \ds1\{\cE_j(\delta)\} \leq 1 - \alpha\right) \\
& \leq \sum_{i=2}^n \P(\cE_i(\delta^{i-1}) \ \text{does not hold}) + \sum_{i=2}^n \P\left(\frac{1}{i-1}\sum_{j=1}^{i-1} \ds1\{\cE_j(\delta)\} \leq 1 - \rho\delta -  (\alpha - \rho\delta) \right) .
\end{align*}
Let $\delta' = \rho\delta$. Note that by a simple Hoeffding's inequality and a union bound one has
$$\P\left(\frac{1}{i-1}\sum_{j=1}^{i-1} \ds1\{\cE_j(\delta)\} \leq 1 - \delta' -  (\alpha - \delta') \right) \leq \min((i-1) \delta', \exp(- 2 (i-1) (\alpha - \delta')^2) ,$$
and thus we obtain with the above calculations
\begin{align*}
& \P\left( \exists \ (i,t) \in \{2, \hdots, n\} \times \{1,\dots\} : \left(1- \delta'-\sqrt{\sqrt{\delta'} \log(1/\delta')}\right) (T_i(t)-1) \geq\kappa \sum_{j\neq i} T_j(t)\right) \\
& \leq \sum_{i=2}^n \left(\delta'^{i-1} + \min((i-1) \delta', \exp(- 2 (i-1) \sqrt{\delta'} \log(1/\delta')))\right) \\
& \leq \sqrt{\delta'} + \frac{2 \delta'}{1-\delta'} = \sqrt{\rho \delta} + \frac{2 \rho \delta}{1-\rho\delta}.
\end{align*}
\end{proof}

Treating $\epsilon$ and factors of $\log\log(\beta)$ as constants, Lemma~\ref{lem:1} says that the total number of times the suboptimal arms are sampled does not exceed $(\beta+2)^2 \left( c_1 \mathbf{H}_1 \log(1/\delta) + c_3 \mathbf{H}_3 \right)$. Lemma~\ref{lem:2} states that only the optimal arm will meet the stopping condition with $a=c_a \left( \frac{2+\beta}{\beta} \right)^2$. Combining these results, we observe that the total number of times all the arms are sampled does not exceed $(\beta+2)^2 \left( c_1 \mathbf{H}_1 \log(1/\delta) + c_3 \mathbf{H}_3 \right)  \left(1 + c_a \left( \frac{2+\beta}{\beta} \right)^2 \right)$, completing the proof of the theorem. We also observe using the approximation $c_a=1$,  the optimal choice of $\beta \approx 1.66$.

\section{Implementation and Simulations}

In this section we investigate how the state of the art methods for solving the best arm problem behave in practice. Before describing each of the algorithms in the comparison, we briefly describe a LIL-based stopping criterion that can be applied to any of the algorithms.
\begin{description} 
\item \hspace{.25in}\textbf{LIL Stopping (LS)} : For any algorithm and $i \in [n]$, after the $t$-th time we have that the $i$-th arm has been sampled $T_i(t)$ times and accumulated a mean $\hatmui$. We can apply Lemma~\ref{lem:lil} (with a union bound) so that with probability at least  $1-\frac{2+\eps}{\eps} \left(\frac{\delta}{\log(1+\eps)}\right)^{1+\epsilon}$ 
\begin{align*}
 \left| \hatmui - \mu_i \right|  &\leq B_{i,T_i(t)} := (1+\sqrt{\eps}) \sqrt{ \frac{ 2 \sigma^2 (1+\eps)  \log \left(\frac{2n\log((1+\eps) T_i(t) + 2)}{\delta}\right)}{T_i(t)} } 
\end{align*}
for all $t\geq1$ and all $i\in[n]$. We may then conclude that if $\widehat{i} := \arg\max_{i \in [n]} \hatmui$ and $\widehat{\mu}_{\widehat{i} ,T_{\widehat{i}} (t)} - B_{\widehat{i} ,T_{\widehat{i}} (t)} \geq  \widehat{\mu}_{j,T_j(t)} + B_{j,T_j(t)}$ then with high probability we have that $\widehat{i} = i_*$.
\end{description}
The LIL stopping condition is somewhat naive but often quite effective in practice for smaller size problems when $\log(n)$ is negligible. To implement the strategy for any fixed confidence algorithm, simply run the algorithm with $\delta/2$ in place of $\delta$ and assign the other $\delta/2$ confidence to the LIL stopping criterion. The algorithms compared were:
\begin{itemize}
\item {\em Nonadaptive + LS} : Draw a random permutation of $[n]$ and sample the arms in an order defined by cycling through the permutation until the LIL stopping criterion is met.  
\item {\em Exponential-Gap Elimination (+LS)} \cite{karnin2013almost} : This procedure proceeds in stages where at each stage, {\em median elimination}\cite{even2002pac} is used to find a 
$\epsilon$-optimal arm whose mean is guaranteed (with large probability) to be within a specified $\epsilon>0$ of the mean of the best arm, and then arms are discarded if their empirical mean is sufficiently below the empirical mean of the $\epsilon$-optimal arm. The algorithm terminates when there is only one arm that has not yet been discarded (or when the LIL stopping criterion is met). 
\item {\em Successive Elimination}  \cite{even2002pac} : This procedure proceeds in the same spirit as {\em  Exponential-Gap Elimination} except the landmark arm is equal to $\widehat{i} := \arg\max_{i \in [n]} \hatmui$. One observes that the algorithm's usual stopping condition and the LIL stopping criterion are one in the same.   
\item {\em lil'UCB (+LS)} : The procedure of Figure~\ref{fig:lilucb} is run with $\epsilon = 0.01$, $\beta = 1$, $a=(2+\beta)^2/\beta^2 = 9$, and $\delta = \left( \frac{\nu \eps}{5(2+\eps)}\right)^{1/(1+\eps)}$ for input confidence $\nu$. The algorithm terminates according to Fig.~\ref{fig:lilucb} or when the LIL stopping criterion is met.
\item {\em lil'UCB Heuristic} : The procedure of Figure~\ref{fig:lilucb} is run with $\epsilon = 0$, $\beta = 1/2$, $a=1+10/n$, and $\delta = \nu/5$ for input confidence $\nu$. These parameter settings do not satisfy the conditions of Theorem~\ref{th:lilucb}, and thus there is no guarantee that this algorithm will find the best arm.  However, as the experiments show, this algorithm performs exceptionally well in practice and therefore we recommend this lil'UCB algorithm in practice.  The algorithm terminates according to Fig.~\ref{fig:lilucb}. 
\item {\em UCB1 + LS} \cite{auer2002finite} : This is the classical UCB procedure that samples the arm $$\arg\max_{i \in [n]} \ \widehat{\mu}_{i,T_i(t)} + \textstyle\sqrt{ \frac{2 \log( t ) }{ T_i(t) }}$$ at each time $t$ and terminates when the LIL stopping criterion is met. 
\end{itemize}
We did not compare to {\em PRISM} of \cite{jamieson2013finding} because the algorithm and its empirical performance are very similar to {\em Exponential-Gap Elimination} so its inclusion in the comparison would provide very little added value. We remark that the first three algorithms require $O(1)$ amortized computation per time step, the lil'UCB algorithms require $O(\log(n))$ computation per time step using smart data structures\footnote{To see this, note that the sufficient statistic for lil'UCB for deciding the next arm to sample depends only on $\hatmui$ and $T_i(t)$ which only changes for an arm if that particular arm is pulled. Thus, it suffices to maintain an ordered list of the upper confidence bounds in which deleting, updating, and reinserting the arm requires just $O(log(n))$ computation. Contrast this with a UCB procedure in which the upper confidence bounds depend explicitly on $t$ so that the sufficient statistics for pulling the next arm changes for all arms after each pull, requiring $\Omega(n)$ computation per time step.}, and UCB1 requires $O(n)$ computation per time step. Due to the poor computational scaling of UCB1 with respect to the problem size $n$, UCB1 was not run on all problem sizes due to practical time constraints. 

Three problem scenarios were considered over a variety problem sizes (number of arms). The ``1-sparse'' scenario sets $\mu_1=1/4$ and $\mu_i = 0$ for all $i=2,\dots,n$ resulting in a hardness of $\mathbf{H}_1 = 4n$. The ``$\alpha=0.3$'' and ``$\alpha=0.6$'' scenarios consider $n+1$ arms with $\mu_0 = 1$ and $\mu_i = 1 - (i/n)^\alpha$ for all $i=1,\dots,n$ with respective hardnesses of $\mathbf{H}_1 \approx 3/2n$ and $\mathbf{H}_1 \approx 6 n^{1.2}$. That is, the $\alpha=0.3$ case should be about as hard as the sparse case with increasing problem size while the $\alpha=0.6$ is considerably more challenging and grows super linearly with the problem size. See \cite{jamieson2013finding} for an in-depth study of the $\alpha$ parameterization.  All experiments were run with input confidence $\delta = 0.1$. All realizations of the arms were Gaussian random variables with mean $\mu_i$ and variance $1/4$\footnote{The variance was chosen such that the analyses of algorithms that assumed realizations were in $[0,1]$ and used Hoeffding's inequality were still valid using sub-Gaussian tail bounds with scale parameter $1/2$.}.

Each algorithm terminates at some finite time with  high probability so we first consider the relative stopping times of each of the algorithms in Figure~\ref{fig:stoppingTimes}. Each algorithm was run on each problem scenario and problem size 40 times. The first observation is that {\em Exponential-Gap Elimination (+LS)} appears to barely perform better than uniform sampling with the LIL stopping criterion. This confirms our suspicion that the constants in {\em median elimination} are just too large to make this algorithm practically relevant. It should not come as a great surprise that {\em successive elimination} performs so well because even though it is suboptimal in the problem parameters, its constants are  small leading to a practical algorithm. The {\em lil'UCB+LS} and {\em UCB1+LS} algorithms seem to behave  comparably and lead the pack of algorithms with theoretical algorithms. The LIL stopping criterion seems to have a large impact on performance of the regular {\em lil'UCB} algorithm, but it had no impact on the {\em lil'UCB Heuristic} variant (not plotted). While  {\em lil'UCB Heuristic} has no theoretical guarantees of outputting the best arm, we remark that over the course of all of our tens of thousands of experiments, the algorithm never failed to terminate with the best arm.

 \begin{figure}[h]
        \centering
        \begin{subfigure}[b]{0.3\textwidth}
        	\caption*{1-sparse,  $\mathbf{H}_1 = 4 n$}
                \includegraphics[width=\textwidth]{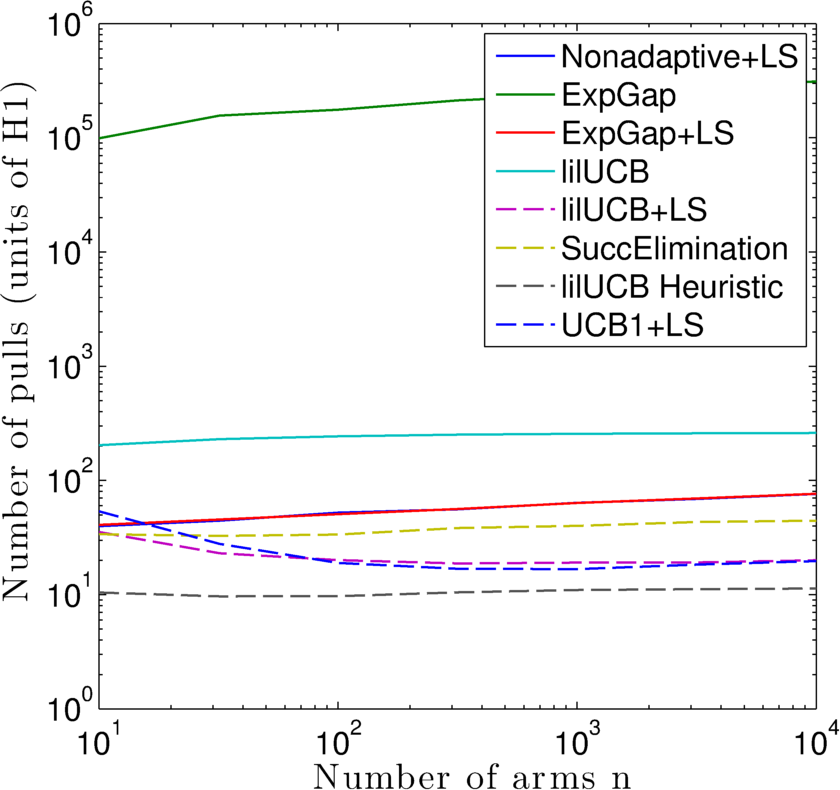}
        \end{subfigure}%
        ~ 
        \begin{subfigure}[b]{0.3\textwidth}
        \caption*{$\alpha=0.3$,  $\mathbf{H}_1 \approx \frac{3}{2} n$}
                \includegraphics[width=\textwidth]{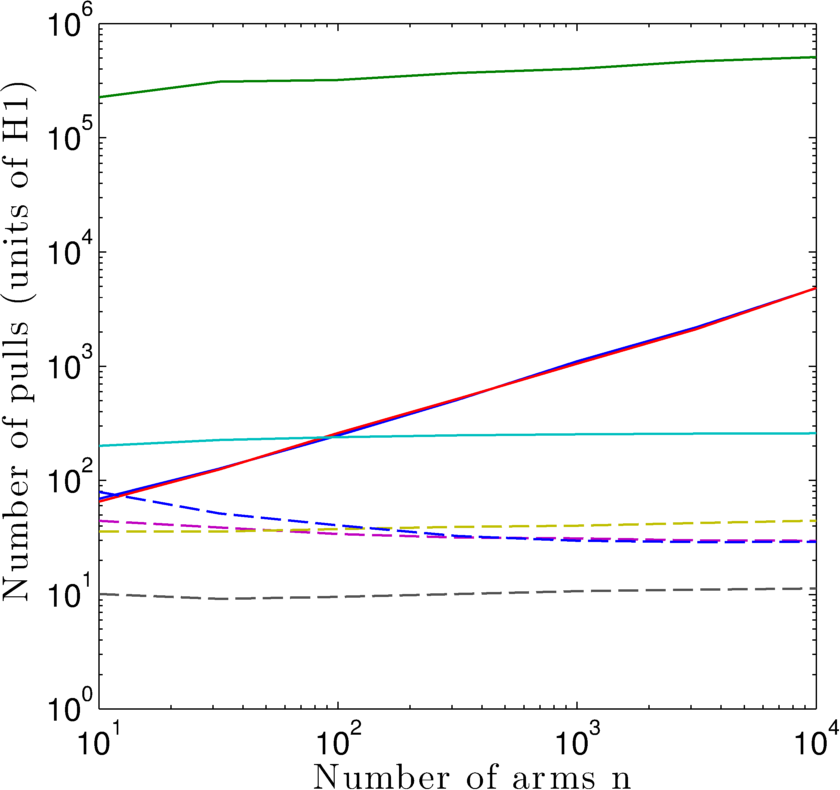}
        \end{subfigure}
        ~ 
        \begin{subfigure}[b]{0.3\textwidth}
        \caption*{$\alpha=0.6$,  $\mathbf{H}_1 \approx 6 n^{1.2}$}
                \includegraphics[width=\textwidth]{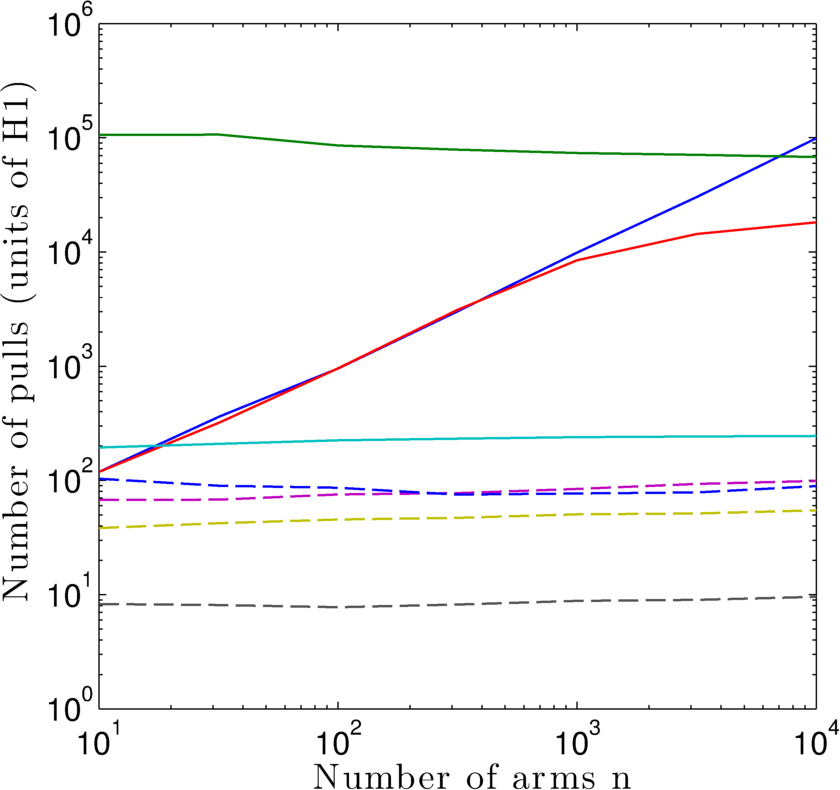}
        \end{subfigure}
    \caption{Stopping times of the algorithms for the three scenarios for a variety of problem sizes. }
    \label{fig:stoppingTimes}
\end{figure}  
        
In reality, one cannot always wait for an algorithm to run until it terminates on its own so we now explore how the algorithms perform if the algorithm must output an arm at every time step before termination (this is similar to the setting studied in \cite{BMS09}). For each algorithm, at each time we output the arm with the highest empirical mean. Because the procedure for outputting the arm is the same across algorithms, measuring how often this output arm is the best arm is a measure of how much information is being gathered by the algorithm's sampling procedure. Clearly in the beginning, the probability that a sub-optimal arm is output by any algorithm is very close to 1. However, as time increases we know that this probability of error should decrease to at least the desired input confidence, and likely, to zero. Figure~\ref{fig:anyTimePlot} shows the ``anytime'' performance of the algorithms for the three scenarios and unlike the empirical stopping times of the algorithms, we now observe large differences between the algorithms. Each experiment was repeated 5000 times. Again we see essentially no difference between nonadaptive sampling and the exponential-gap procedure. While in the stopping time plots of Figure~\ref{fig:stoppingTimes} {\em successive elimination} appears neck-and-neck with the UCB algorithms, we observe in Figure~\ref{fig:anyTimePlot} that the UCB algorithms are collecting sufficient information to output the best arm at least twice as fast as {\em successive elimination}. This tells us that the stopping conditions for the UCB algorithms are still too conservative in practice which motivates the use of the {\em lil'UCB Heuristic} algorithm which appears to perform very strongly across all metrics.

\begin{figure}[H]
        \centering
        \begin{subfigure}[b]{0.01\textwidth}
                \begin{turn}{90}$n=10$\end{turn}\vspace{.6in}
        \end{subfigure}
        ~
        \begin{subfigure}[b]{0.3\textwidth}
        	\caption*{1-sparse,  $\mathbf{H}_1 = 4 n$}
                \includegraphics[width=\textwidth]{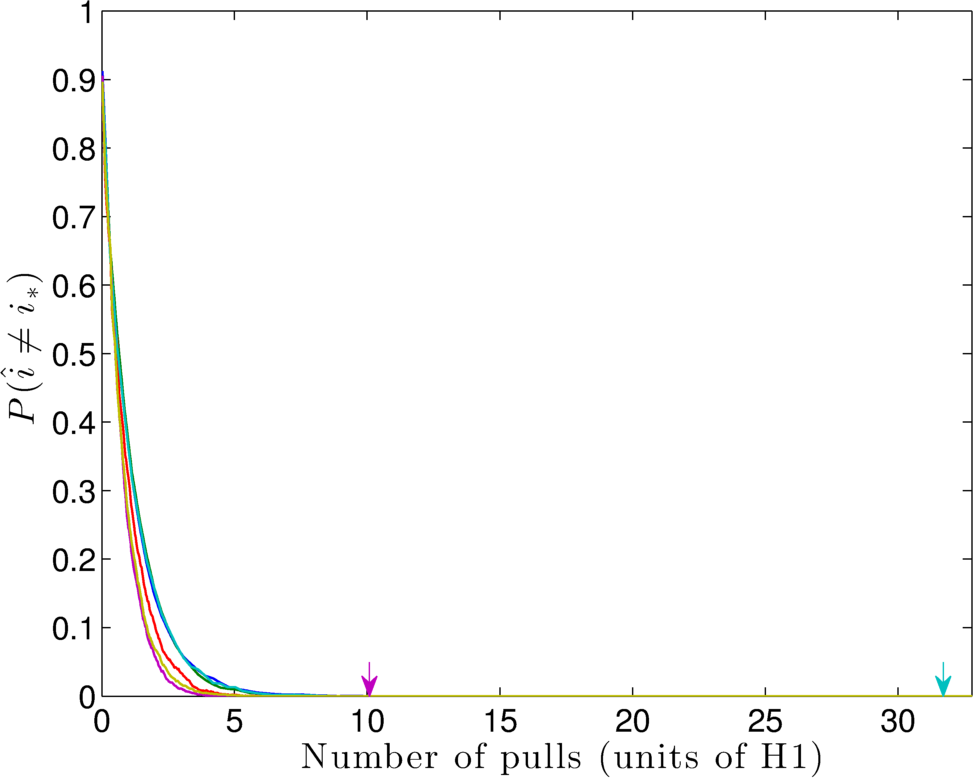}
        \end{subfigure}%
        ~ 
        \begin{subfigure}[b]{0.3\textwidth}
        \caption*{$\alpha=0.3$,  $\mathbf{H}_1 \approx \frac{3}{2} n$}
                \includegraphics[width=\textwidth]{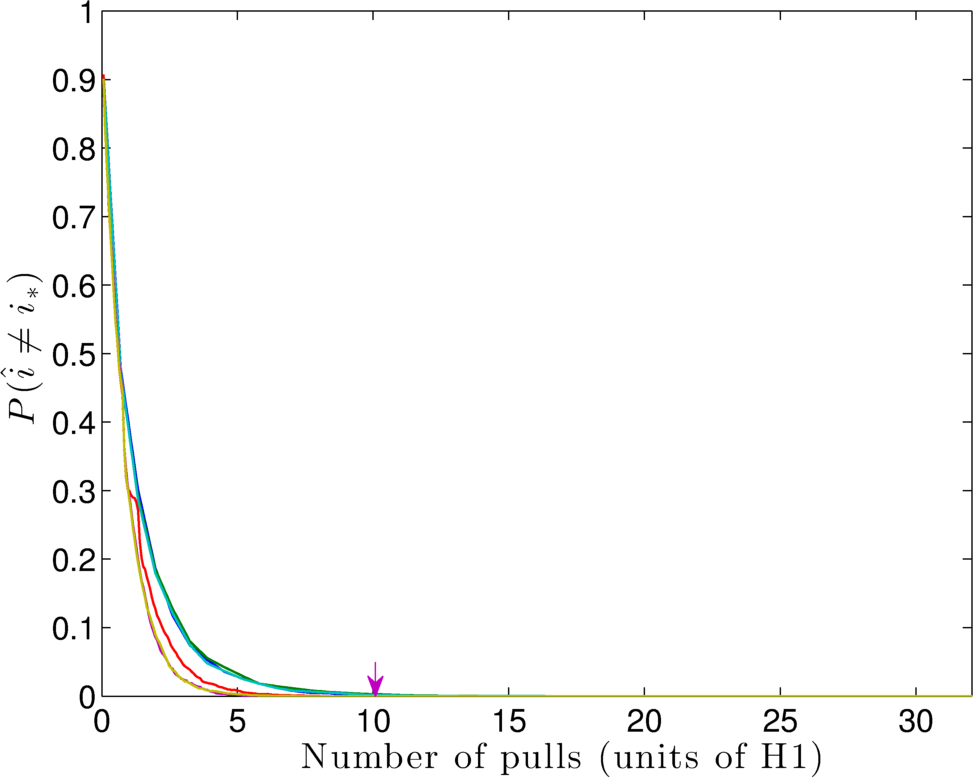}
        \end{subfigure}
        ~ 
        \begin{subfigure}[b]{0.3\textwidth}
        \caption*{$\alpha=0.6$,  $\mathbf{H}_1 \approx 6 n^{1.2}$}
                \includegraphics[width=\textwidth]{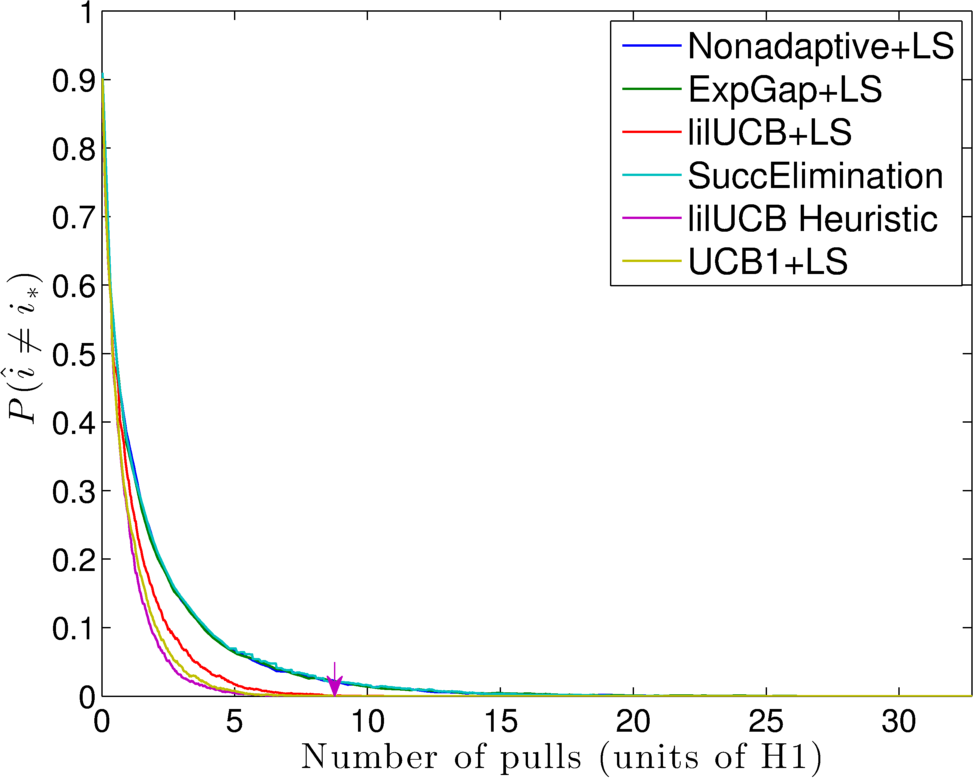}
        \end{subfigure}
        
        \begin{subfigure}[b]{0.01\textwidth}
                \begin{turn}{90}$n=100$\end{turn}\vspace{.55in}
        \end{subfigure}
        ~
        \begin{subfigure}[b]{0.3\textwidth}
                \includegraphics[width=\textwidth]{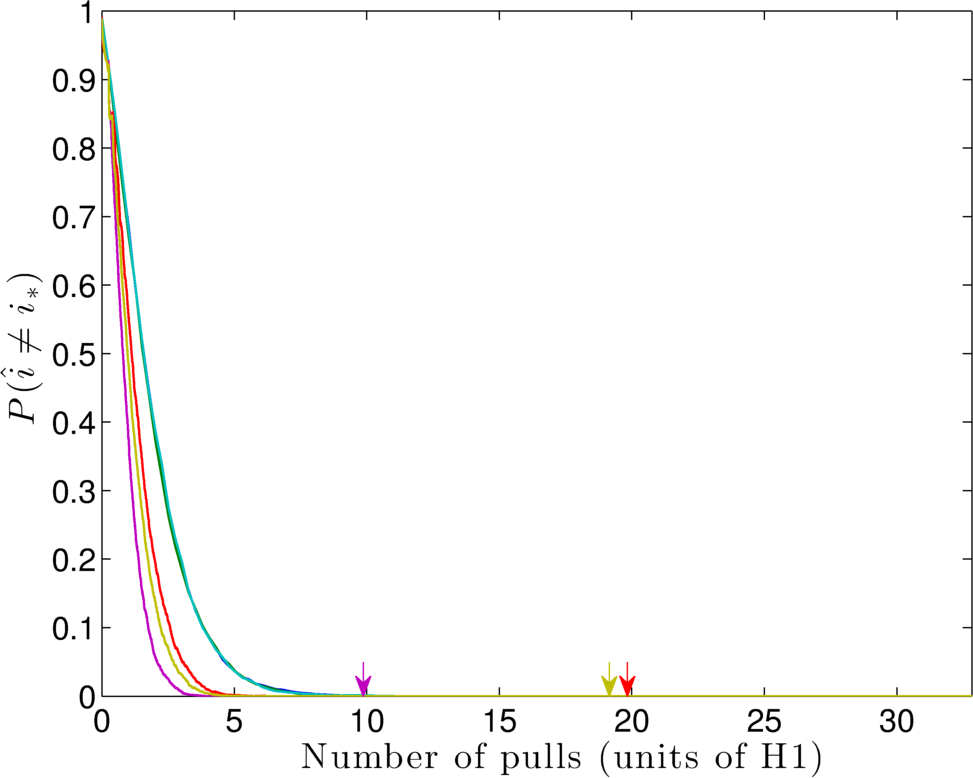}
        \end{subfigure}%
        ~ 
        \begin{subfigure}[b]{0.3\textwidth}
                \includegraphics[width=\textwidth]{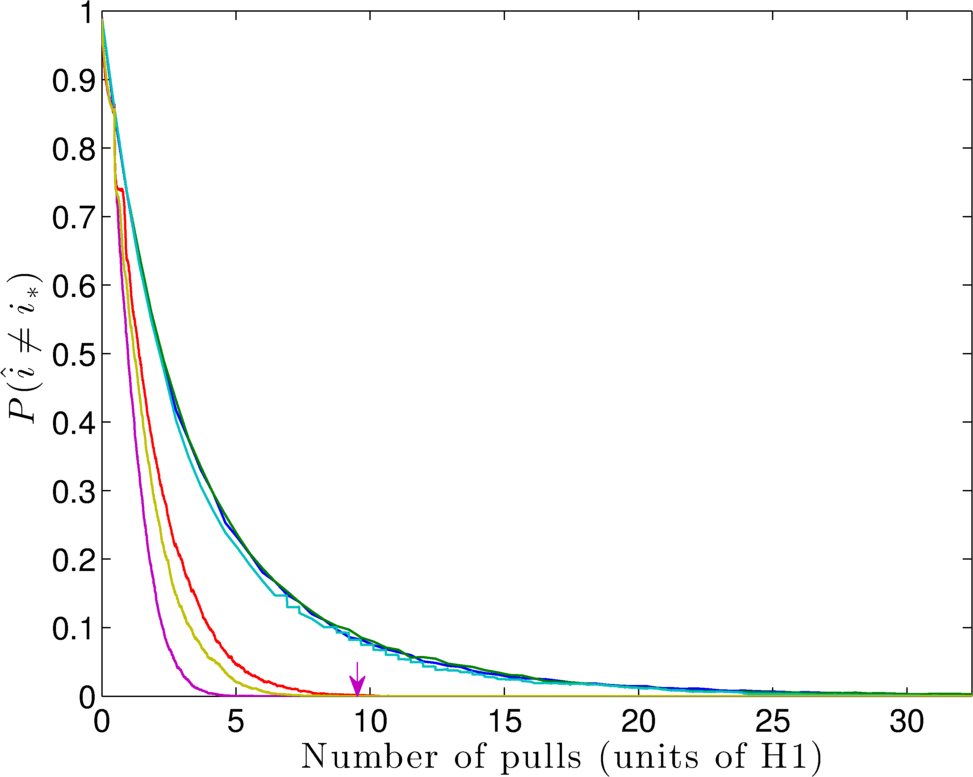}
        \end{subfigure}
        ~ 
        \begin{subfigure}[b]{0.3\textwidth}
                \includegraphics[width=\textwidth]{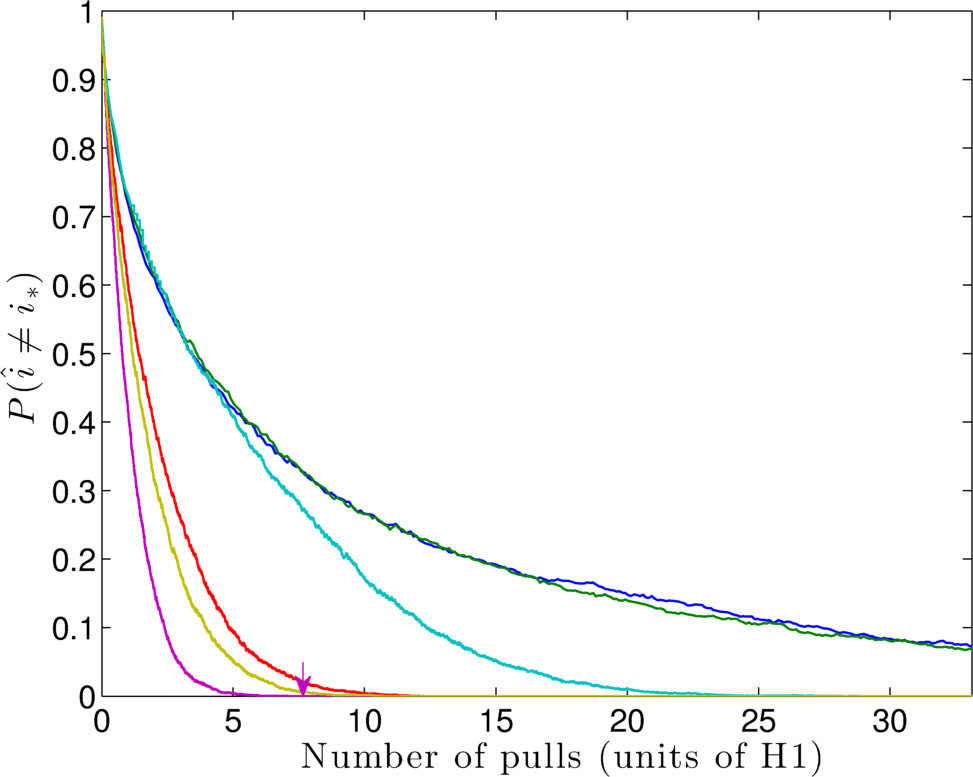}
        \end{subfigure}
        
        \begin{subfigure}[b]{0.01\textwidth}
                \begin{turn}{90}$n=1000$\end{turn}\vspace{.5in}
        \end{subfigure}
        ~
         \begin{subfigure}[b]{0.3\textwidth}
                \includegraphics[width=\textwidth]{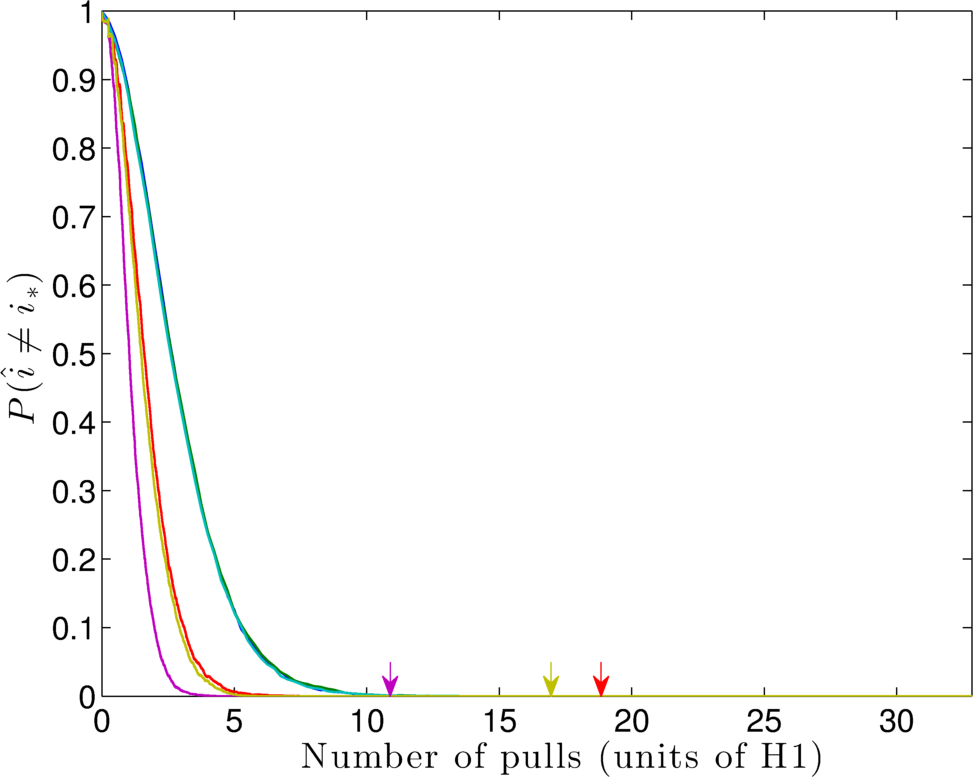}
        \end{subfigure}%
        ~ 
        \begin{subfigure}[b]{0.3\textwidth}
                \includegraphics[width=\textwidth]{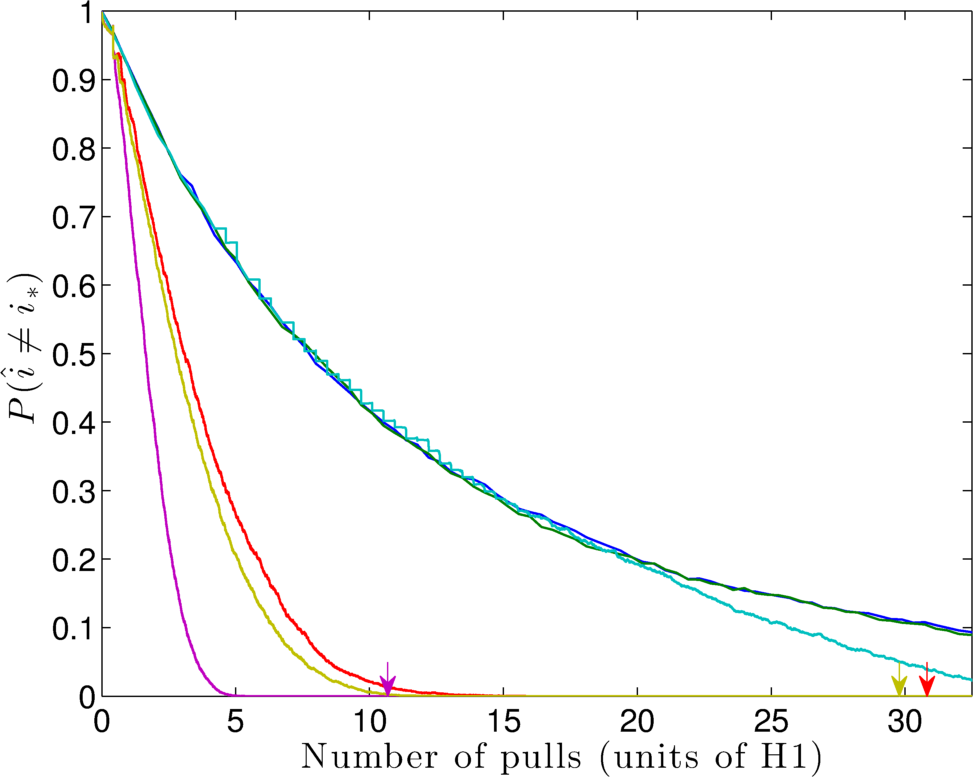}
        \end{subfigure}
        ~ 
        \begin{subfigure}[b]{0.3\textwidth}
                \includegraphics[width=\textwidth]{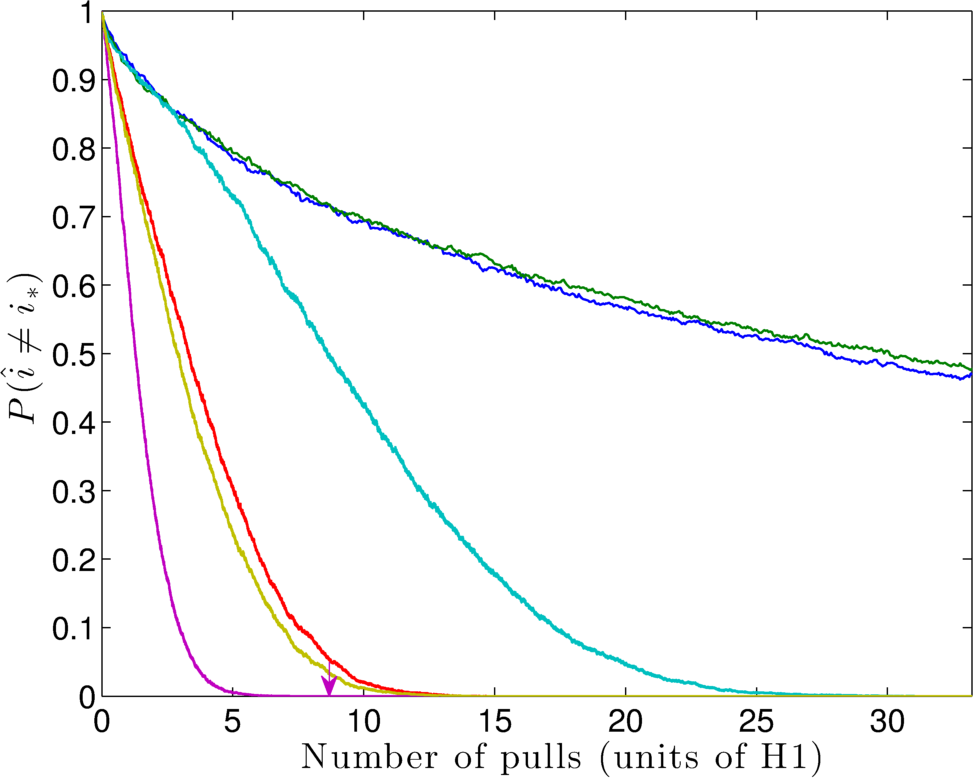}
        \end{subfigure}
        
        \begin{subfigure}[b]{0.01\textwidth}
                \begin{turn}{90}$n=10000$\end{turn}\vspace{.45in}
        \end{subfigure}
        ~
         \begin{subfigure}[b]{0.3\textwidth}
                \includegraphics[width=\textwidth]{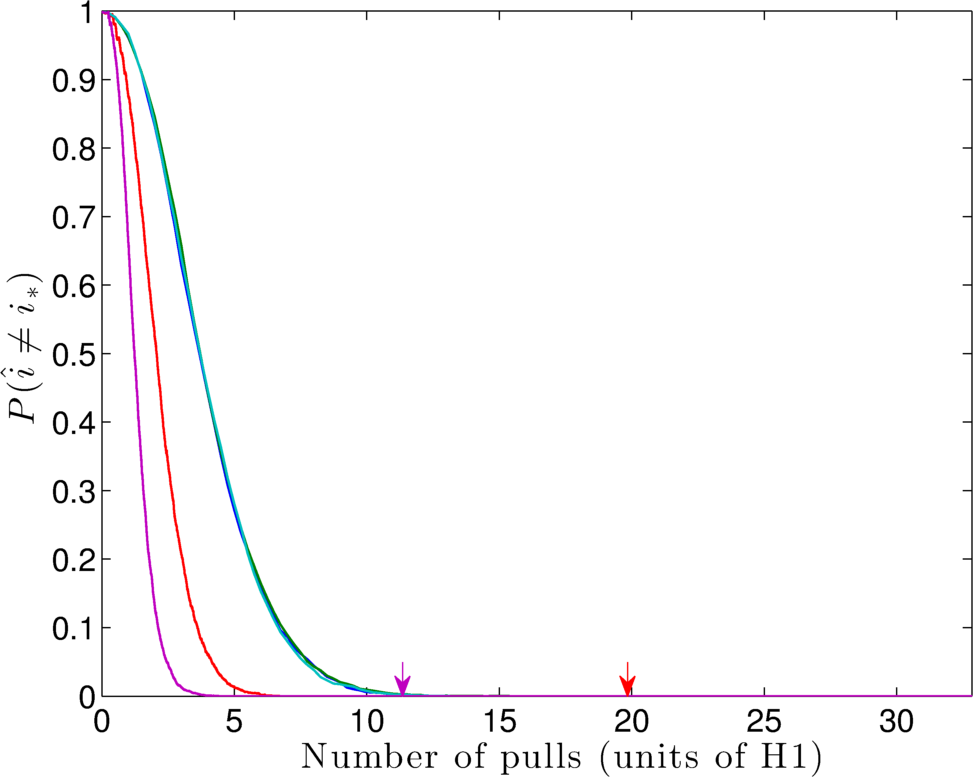}
        \end{subfigure}%
        ~ 
        \begin{subfigure}[b]{0.3\textwidth}
                \includegraphics[width=\textwidth]{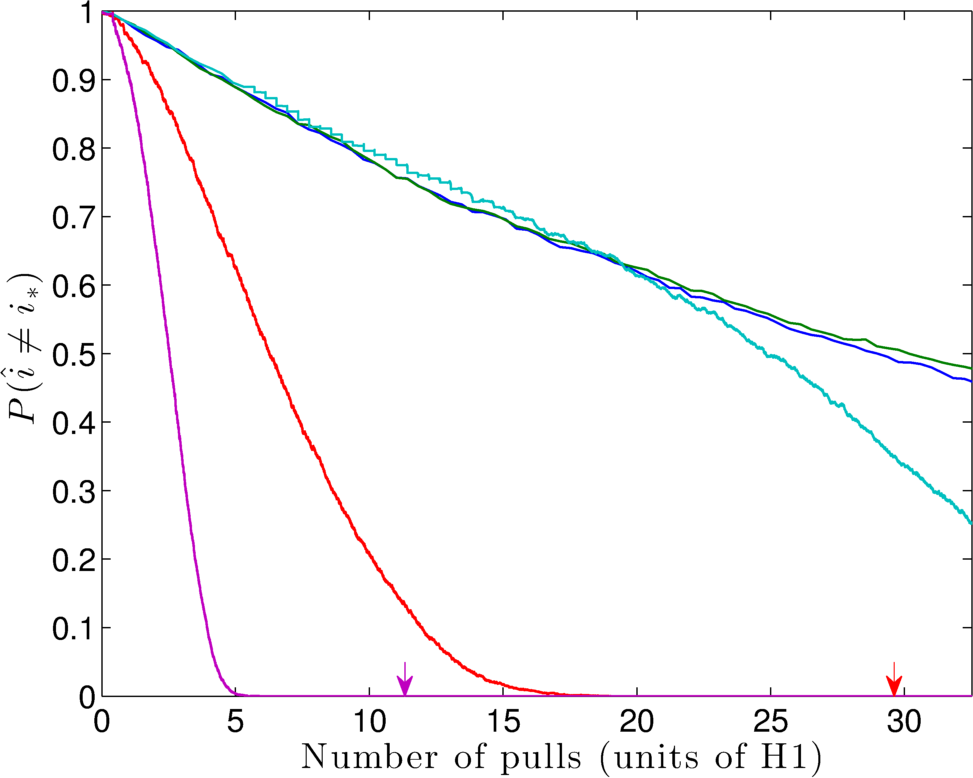}
        \end{subfigure}
        ~ 
        \begin{subfigure}[b]{0.3\textwidth}
                \includegraphics[width=\textwidth]{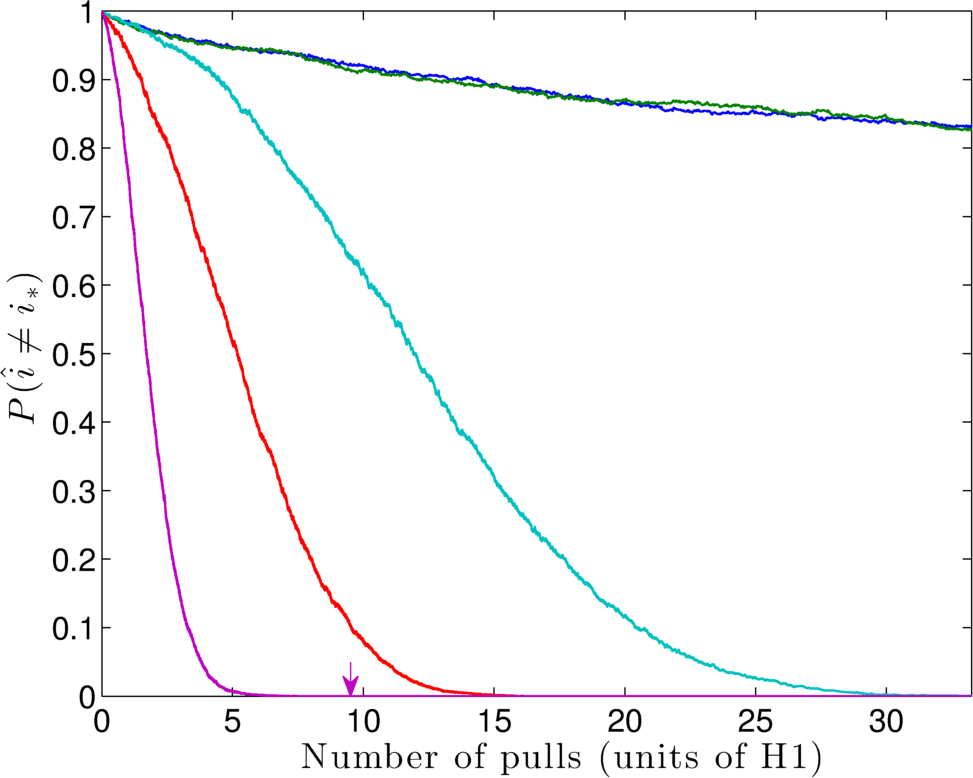}
        \end{subfigure}
        
        \caption{At every time, each algorithm outputs an arm $\hat{i}$ that has the highest empirical mean. The $\P(\hat{i} \neq i_*)$ is plotted with respect to the total number of pulls by the algorithm. The problem sizes (number of arms) increase from top to bottom. The problem scenarios from left to right are the 1-sparse problem ($\mu_1=0.5$, $\mu_i = 0 \ \forall i>1$) , $\alpha=0.3$ ($\mu_i = 1  - ({i}/{n})^{\alpha}$, $i=0,1,\dots,n$), and $\alpha=0.6$. The arrows indicate the stopping times (if not shown, those algorithms did not terminate within the time window shown). Note that UCB1 is not plotted for $n=10000$ due to computational constraints. Also note that in some plots it is difficult to distinguish between the nonadaptive sampling procedure, the exponential-gap algorithm, and successive elimination due to the curves being on top of each other. }
        \label{fig:anyTimePlot}
\end{figure}

\clearpage
\bibliographystyle{unsrt}
\bibliography{LilUCB}

\newpage
\appendix

\section{Condensed Proof of Lower Bound}  \label{app:LB}
In the following we show a weaker result than what is shown in \cite{Farrell}; nonetheless, it shows the $\log \log$ term is necessary.  
\begin{theorem} Let $X_i \overset{i.i.d.}{\sim} \mathcal{N}(\Delta,1)$, where $\Delta\neq 0$ is unknown.  Consider testing whether $\Delta>0$ or $\Delta <0$. Let $Y \in \{-1,1\}$ be the decision of any such test based on $T$ samples (possibly a random number).  If $\sup_{\Delta\neq 0}  \P(Y\neq \mathrm{sign}(\Delta) ) < 1/2$, 
then 
\begin{eqnarray} \nonumber
\limsup_{\Delta\rightarrow 0} \frac{\E[T]}{\Delta^{-2} \log \log {\Delta^{-2}} } &>& 0 \ .
\end{eqnarray}
\end{theorem}

We rely on two intuitive facts, each which justified more formally in \cite{Farrell}.

\begin{description}
\item[Fact 1.] The form of \emph{an} optimal test is a \emph{generalized sequential probability ratio test} (GSPRT), which continues sampling while 
\begin{eqnarray} \nonumber
-B_t \leq \sum_{j=1}^t X_i \leq B_t
\end{eqnarray}
and stops otherwise, declaring $\Delta > 0$ if $\sum_{j=1}^t X_j \geq B_t$, and  $\Delta < 0$ if $\sum_{j=1}^t X_j \leq -B_t$ where $B_t>0$ is non-decreasing in $t$.  This is made formal in \cite{Farrell}.
\item[Fact 2.] 
\label{lem:LIL}
If
\begin{eqnarray} \label{eqn:threshh}
\lim_{t \rightarrow \infty} \frac{B_t}{\sqrt{2 t \log \log t}} \leq 1
\end{eqnarray}
then $Y$, the decision output by the GSPRT, satisfies $\sup_{\Delta \neq 0} \P_{\Delta}( Y\neq \mathrm{sign}{\ \Delta})  = 1/2$.   
This follows from the LIL and a continuity argument (and note the limit exists as $B_t$ is non-decreasing). Intuitively, if the thresholds satisfy (\ref{eqn:threshh}), a zero mean random walk will eventually hit either the upper or lower threshold. The upper threshold is crossed first with probability one half, as is the lower.  By arguing that the error probabilities are continuous functions of $\Delta$, one concludes this assertion is true.
\end{description}

The argument proceeds as follows.  If (\ref{eqn:threshh}) is holds, then the error probability is $1/2$.  So we can focus on threshold sequences satisfying $\lim_{t \rightarrow \infty} \frac{B_t}{\sqrt{2 t \log \log t}} \geq (1+\epsilon)$ for some $\epsilon >0$. In other words, for all $t > t_1$ some $\epsilon > 0$, some sufficiently large $t_1$
\begin{eqnarray*} \label{eqn:asm}
B_t  \geq (1+ \epsilon) \sqrt{2 t \log \log t}.
\end{eqnarray*}
Define the function 
\begin{eqnarray*}
t_0(\Delta) = \frac{\epsilon^2 \Delta^{-2}}{2} \  \log \log\left(\frac{\Delta^{-2}}{2}\right)
\end{eqnarray*} 
and let $T$ be the stopping time:
\begin{eqnarray*} \nonumber
T := \inf \left \{t \in \mathbb{N} :  \left \vert \sum_{i=1}^t X_i \right \vert \geq B_t \right \}.
\end{eqnarray*}    
Let $S_t^{(\Delta)} = \sum_{j=1}^t X_j$ for $X_j \overset{iid}{\sim} \mathcal{N}(\Delta,1)$.  Without loss of generality, assume $\Delta >0$. Additionally, suppose $\Delta$ is sufficiently small, such that both $t_0(\Delta) > t_1(\epsilon)$ and $\Delta \leq \epsilon$ (in the following steps we consider the limit as $\Delta \rightarrow 0$).   We have
\begin{eqnarray} \nonumber
\P_{\Delta}  (T \geq t_0(\Delta) ) \hspace{-2.5cm} && \\ \nonumber
 &=&   \P \left( \bigcap_{t=1}^{t_0(\Delta) -1} |S_t^{(\Delta) }| < B_t  \right)  \\ \nonumber
&=&  \  \P \left( \bigcap_{t=1}^{t_1(\epsilon)} \{|S_t^{(\Delta)}| < B_t \} \cap \bigcap_{t=t_1(\epsilon)+1}^{t_0(\Delta) - 1} \{ S_t^{(0)}  < B_t -\Delta t \} \cap \{ S_t^{(0)}  > - B_t -\Delta t \}   \right) \\  
& \geq &  \  \P \left( \bigcap_{t=1}^{t_1(\epsilon)} \{|S_t^{(\Delta)}| < B_t \}  \cap \bigcap_{t=t_1(\epsilon)+1}^{t_0(\Delta) -1} \{ | S_t^{(0)} | < (1+ \epsilon/2) \sqrt{2t \log \log t} \} \right)  \label{eqn:lasts} \\ \nonumber
&=&  \P \left( \bigcap_{t=1}^{t_1(\epsilon)}  |S_t^{(\Delta)} | < B_t  \right) \P \left(  \bigcap_{t=t_1(\epsilon)+1}^{t_0(\Delta) -1} |S_t^{(0)} | \leq (1+ \epsilon/2) \sqrt{2t \log \log t}   \left \vert \bigcap_{t=1}^{t_1(\epsilon)}  |S_t^{(0)} | < B_t \right. \right) \\ \label{eqn:lasts2}
 &\geq&  \P \left( \bigcap_{t=1}^{t_1(\epsilon)} |S_t^{(\Delta)} | < B_t  \right) \P \left(  \bigcap_{t=t_1(\epsilon)+1}^{\infty} |S_t^{(0)} | < (1+ \epsilon/2) \sqrt{2t \log \log t}    \right)
\end{eqnarray}
where (\ref{eqn:lasts}) holds when $\epsilon \geq \Delta$ and (\ref{eqn:lasts2}) holds by removing the conditioning, and then by increasing the number of terms in the intersection.  To see that (\ref{eqn:lasts}) holds, note that  $\frac{2\log \log t}{t} \geq \left( \frac{2\Delta}{\epsilon}\right)^2$ for all $t \leq t_0(\Delta)$, which is easily verified when $ \epsilon \geq \Delta$ since
\begin{eqnarray*}  
\frac{ \log \log \left( \frac{\epsilon^2 \Delta^{-2}}{2} \  \log \log\left(\frac{\Delta^{-2}}{2}\right) \right)}{ \  \log \log\left(\frac{\Delta^{-2}}{2}\right)} &\geq& 1.
\end{eqnarray*}
Taking the limit as $\Delta \rightarrow 0$, for any $\epsilon >0$, gives
\begin{eqnarray*}
\lim_{\Delta \rightarrow 0} \P_{\Delta}  (T \geq t_0(\Delta) )  & \geq & c(\epsilon) >0
\end{eqnarray*}
which follows from (\ref{eqn:lasts2}), as the first term is non-zero  for any $\Delta$ (including $\Delta = 0$) since $t_1(\epsilon)< \infty$ and $B_t>0$, and the second term is non-zero by the LIL  for any $\epsilon >0$.  Note that a finite bound on the second term can be obtained as in Section \ref{th:lilucb}.

 By Markov,  $ \E_{\Delta}[T] / t_0(\Delta) \geq  \P_{\Delta} (T\geq t_0(\Delta)  ) $, and we conclude
\begin{eqnarray*}
\lim_{\Delta \rightarrow 0} \frac{\E_{\Delta}[T]}{\Delta^{-2} \log \log \Delta^{-2}} \geq \epsilon^2 \  c(\epsilon)  > 0
\end{eqnarray*}
for any test with $\sup_{\Delta \neq 0} \P(Y \neq \mbox{sign}(\Delta)) < 1/2$.

\end{document}

%% file: Commands.tex
\newtheorem{lemma}{Lemma}
\newtheorem{corollary}{Corollary}

\renewcommand{\phi}{\varphi}

\renewcommand{\P}{\mathbb{P}}
\newcommand{\E}{\mathbb{E}}

\newcommand{\R}{\mathbb{R}}

\newcommand{\cE}{\mathcal{E}}

\def\ds1{\mathds{1}}
\renewcommand{\epsilon}{\varepsilon}
\newcommand{\eps}{\varepsilon}

\newcommand{\argmax}{\mathop{\mathrm{argmax}}}

\newlength{\minipagewidth}
\newcommand{\beq}{\begin{equation}}
\newcommand{\eeq}{\end{equation}}

\newcommand{\beqa}{\begin{eqnarray}}
\newcommand{\eeqa}{\end{eqnarray}}

\newcommand{\beqan}{\begin{eqnarray*}}
\newcommand{\eeqan}{\end{eqnarray*}}

\def\ba#1\ea{\begin{align*}#1\end{align*}} 
\def\banum#1\eanum{\begin{align}#1\end{align}} 